\documentclass[12pt]{l4dc2021} 

\title[Robust pruning and quantisation bounds]{Robust error bounds for quantised and pruned neural networks}
\usepackage{times}
\usepackage{amsmath,bm}
\usepackage{bbding}
\usepackage{mathtools}
\usepackage{amsmath,amsfonts,mathtools}
\usepackage{pifont}
\usepackage{wasysym}
\usepackage{amssymb}
\usepackage{times}
\usepackage{amsmath}
\usepackage{mathtools}

\DeclarePairedDelimiter\floor{\lfloor}{\rfloor}
\usepackage{amssymb}
\usepackage{graphicx}
\usepackage{float}
\usepackage{graphics}
\usepackage{caption}
\usepackage{url}
\usepackage{setspace}
\usepackage{bytefield}
\usepackage{wrapfig}
\usepackage{caption}
\newtheorem{problem}{Problem}

\makeatletter
 \let\Ginclude@graphics\@org@Ginclude@graphics 
\makeatother

\author{\Name{Jiaqi Li} \Email{jiaqi.li@st-hughs.ox.ac.uk}\\
 \Name{Ross Drummond} \Email{ross.drummond@eng.ox.ac.uk}\\
 \Name{Stephen R. Duncan} \Email{stephen.duncan@eng.ox.ac.uk}\\
 \addr Department of Engineering Science, University of Oxford, 17 Parks Road, OX1 3PJ, Oxford, United Kingdom}

\begin{document}

\maketitle
\vspace{-1.2cm}
\begin{abstract}
%\setstretch{3}

With the rise of smartphones and the internet-of-things, data is increasingly getting generated at the edge on local, personal devices. For privacy, latency and energy saving reasons, this shift is causing machine learning algorithms to move towards decentralisation with the data and algorithms stored, and even trained, locally on devices. The device hardware becomes the main bottleneck for model capability in this set-up, creating a need for slimmed down, more efficient neural networks. Neural network pruning and quantisation are two methods that have been developed for this, with both approaches demonstrating impressive results in reducing the computational cost without sacrificing significantly on model performance. However, the understanding behind these reduction methods remains underdeveloped. To address this issue, a semi-definite program  is introduced to bound the worst-case error caused by pruning or quantising a neural network. The method can be applied to many neural network structures and nonlinear activation functions with the bounds holding robustly for all inputs in specified sets. It is hoped that the computed bounds will provide certainty to the performance of these algorithms when deployed on safety-critical systems.

\end{abstract}


\section{Introduction}
%\setstretch{3}

The way we generate data is rapidly changing in response to the rise of the  internet-of-things and the widespread use of smartphones \citep{zhou}. In terms of data sources, there is an accelerating shift away from  big data, like e-commerce records, stored on mega-scale cloud data centres towards more localised data sources generated by users interacting with their personal devices \citep{zhang}. This evolving data landscape is causing a rethink to how machine learning algorithms are trained and implemented. % with potentially major changes to the current approach being needed. 

The most intuitive way to train and implement these algorithms within this new data landscape remains \emph{centralised learning} where the data is sent to a central data centre in the cloud where the algorithm is then trained and run \citep{zhou}. This approach should give the best model performance and is a natural extension of existing algorithm training frameworks but has several flaws \citep{zhou,gunduz2019machine}. These include: (i) Concerns around user privacy since the data is stored in a data centre which could be owned by a third party, (ii) The increased latency of centralised algorithms in responding to new data, and (iii.) As the algorithms get ever larger in response to the growing availability of data, they are requiring significantly more compute power, memory storage and energy. 

These flaws are inherent to centralised learning and motivate growing efforts to direct machine learning towards a more \emph{decentralised} approach. With decentralised learning \citep{zhang,zhou} (often known as \emph{edge intelligence} \citep{zhou, deng2020edge,shi2020communication} and falling under the umbrella of  \emph{distributed learning} \citep{zhang} amongst other labels), there is a greater emphasis on implementing and even training the algorithms locally on devices. In this way, decentralised learning aims to generate smaller, more personalised algorithms that can help circumnavigate the above-mentioned issues with centralised learning- albeit typically doing so at the expense of model performance \citep{zhang}. Running the algorithms locally on the devices means that hardware constraints (e.g. memory availability, compute power and energy consumption) become important. Managing these constraints whilst still delivering reasonable model performance is one of the main bottlenecks that decentralised learning needs to overcome and has led to a push for slimmed-down, efficient algorithms (typically neural networks (NNs)- the class of model focussed on in this work). As such, there has been increased interest in pruned neural networks which sparsify the weights and biases \citep{blalock} and quantised neural networks (e.g. \citep{shin2016fixed,courbariaux2014training,sung2015resiliency}) which use a coarse fixed-point representation of the algorithm's parameters. However, a full and robust understanding of how modifications like network pruning and quantisation affect the neural networks outputs, that goes beyond sensitivity studies like \citep{shin2016fixed} for instance, is as yet still not fully developed. %This is an issue of increasing importance as these algorithms get applied within safety-critical systems like automobiles.

To address these issues, this paper introduces a formal framework to compute performance guarantees for neural networks implemented on low-cost and memory-limited hardware. Specifically, we address the problem of determining how the output of a neural network is corrupted by either quantising its parameters and/or data after storing them using fixed-point arithmetic or from pruning it. In other words, a method to compute bounds for quantised and pruned neural networks is introduced.   The main contributions are: 
\begin{itemize}
\item A semi-definite program (SDP) to bound the worst-case error from neural network pruning or from quantising its weights, biases and input using fixed-point arithmetic. %As such, robust error bounds for pruned and quantised neural networks are introduced. 
\item The computed bounds hold robustly for all input data in specified sets and account for the neural network's nonlinearities.
\item The neural network quantisation and pruning bounds are shown to be special cases of a more general problem, called the neural network \emph{similarity problem}, which bounds the worst-case error of two different neural networks outputs for all inputs in specific sets.
%\item The application of the bounds are explored in several numerical examples.
\end{itemize}
The neural network \emph{similarity problem} (Problem \ref{prob:general}) connects to the notion of incremental stability in dynamical systems \citep{zames}, with the second neural network of the problem being either the quantised or pruned version.

%This result allows us to guarantee the worst-case error following from the corruption of a trained neural network by pruning it or implementing it on low-cost, memory-limited hardware with fixed-point arithmetic, as in for quantised neural networks. To arrive at this bound, a more general problem is first introduced in Problem \ref{prob:general} that bounds the worst-case error between the outputs of two generic neural networks. This general problem is labelled the neural network \emph{similarity problem} and draws heavily from the notion of incremental stability in dynamical systems \citep{zames}.  By setting the second neural network of this problem to be the quantised or pruned one, the desired error bounds are obtained as a special case. 

\paragraph{Related work:} Many results have demonstrated the potential of both quantised and pruned neural networks to realise machine learning on limited hardware.  For example, \cite{gong2014compressing} achieved a 16-24$\times$ network compression for the 1000-category classification on ImageNet with only a 1$\%$ loss of classification accuracy and \cite{lin2016fixed} showed no loss in accuracy by reducing a model trained on the CIFAR-10 benchmark by $>20\%$. More extreme levels of quantisation have also been explored, for example binarised neural networks (where the weights and/or activation functions are restricted to be binary) have also demonstrated rapid compute speeds and large memory savings without significant accuracy losses (\cite{hubara2016binarized}  and \cite{rastegari2016xnor}). 

The quantisation error bounds presented in this work are closely related to the semi-definite programmes for certifying neural network robustness developed in  \cite{fazlyab2019safety,raghunathan2018semidefinite,dathathri}. These SDPs guarantee that small perturbations in the input will not lead to large variations in the NN output and can help prevent against adversarial attacks and make the NNs more reliable to out-of-sample data. Recently, there has been a focus on improving the scalability of these SDP robustness certificates \citep{dathathri} (since they scale by $\mathcal{O}(n^6)$ in runtime and $\mathcal{O}(n^4)$ in memory requirements, where $n$ is the number of neurons in the network \citep{dathathri}), a particularly exciting direction for this paper as it could help make the presented bounds applicable to larger NNs. 

%%%%%%%%%%%%%%%%%%%%%%%%%%%%%%%%%%%%%%%%%%%%%%%%%%%%%
\subsection*{Acronyms, Notations and Preliminaries}

$\mathbb{R}_+$, $\mathbb{R}^n_+$ and $\mathbb{R}^{n \times n}_+$ denote real non-negative numbers, non-negative real vectors of dimension $n$ and real non-negative matrices of size $n \times n$ respectively. Positive diagonal matrices of size $n \times n$ are denoted $\mathbb{D}^n_+$. $\mathbb{Z}_+$ ($\mathbb{Z}_-$) is the set of positive (negative) integers. The set of natural numbers (non-negative integers) is $\mathbb{N}$. The matrix of zeros of dimension $n \times m$  is $\bm{0}_{n \times m}$ and $\bm{0}_{n}$ is the vector of zeros of dimension $n$. The matrix of ones of dimension $N \times M$ is $\bm{1}_{N \times M}$ and $\bm{1}_{N}$ is the vector of ones of dimension $N$. The identity matrix of dimension $N$ is $\mathbb{I}_{N}$.  The $blkdiag\{a\}$ function takes an array as input and returns a block diagonal matrix with elements of $a$ on its main diagonal as ordered in $a$. Strict and elementwise LMIs are posed as $F(x) > 0$ while non-strict and elementwise LMIs are written as $F(x) \geqslant 0$. A positive (negative) definite matrix $\Omega$ is denoted $\Omega \succ (\prec )~ 0$, and a positive (negative) semi-definite $\Gamma$ matrix is written $\Gamma \succcurlyeq (\preccurlyeq )~ 0$. The $p$-norm is displayed by $||\cdot||_p:\mathbb{R}^n\to\mathbb{R}_+, p \in \mathbb{N}$. Where acronyms are used, ``NN'' stands for ``artificial neural network'', ``QC'' stands for quadratic constraint, ``LMI'' stands for ``linear matrix inequality'', and ``SDP'' stands for ``semi-definite programme''.

%-------------------------------------------------------------------------------------
\section{Neural Networks: Representation in state-space form} \label{sec:genProblem}

Consider two functions $f_1(x_1): \mathcal{X} \to \mathcal{F}$ and $f_2(x_2): \mathcal{X} \to \mathcal{F}$ acting upon input data $x_1, x_2 \in \mathcal{X}$. These functions are assumed to be defined by feed-forward NNs composed of $\ell$ hidden layers, with the $k^{\text{th}}$ layer  composed of $n_k$ neurons and the total number of neurons being $N = \sum_{k = 1}^{l} n_k$. Both the first and second neural network ($i \in \{1,2\}$) can be expanded out in a state-space-like form
\begin{subequations}\label{nn1}\begin{align}
x_i^0  & = x_i,
\\
x_i^{k+1}  & = \phi(W_i^kx_i^k + b_i^k), \quad k= 0, 1, \dots, \ell-1,
\\
f_i(x_i)  & = W_i^\ell x_i^\ell + b_i^\ell.
\end{align}\end{subequations}
% and second neural network
% \begin{subequations}\label{nn2}\begin{align}
% x_2^0  & = x_2,
% \\
% x_2^{k+1}  & = \phi(W_2^kx^k_2 + b_2^k), \quad k= 0, 1, \dots, \ell-1,
% \\
% f_2(x_2)  & = W_2^\ell x^\ell_2 + b_2^\ell,
% \end{align}\end{subequations}
Here, for neural network $i \in \{1,2\}$, $x^0_i \in \mathcal{X}_i \subseteq \mathbb{R}^{n_x}$ represents the input data, $W^k_i\in \mathbb{R}^{n_{k+1} \times n_k}$ the weights, $b^k_i \in \mathbb{R}^{n_{k+1}}$ the biases, $ k = 0,\, \dots, \ell-1$, defines the network layers,  $n_x$ and $n_f$ are, respectively, the network's input and output dimensions  and $\phi(\cdot)$ are the activation functions- typically a rectified Linear Unit (ReLU), sigmoid or tanh- which in this paper can be any function satisfying some of the properties of Definition \ref{def:nonlinearfunctional} in Appendix 4. The activation functions are taken to act elementwise on their vector arguments.

%\begin{remark}
%To simplify the notation, this work focuses on feed-forward NNs but could be generalised to other architectures that can be written in the above state-space like form, such as recurrent neural networks. In a similar vein, it will also be assumed that each NN in \eqref{nn1}-\eqref{nn2} will have the same layer dimensions and activation functions. Again, this restriction can also be relaxed at the expense of a more complex notation.
%\end{remark}

\subsection{Compact neural network representation}

For the $i^\text{th}\in \{1,\, 2\}$ neural network, define the arguments of the activation functions  $\phi(\cdot)$ as
 \begin{align}\label{xi_def}
    \xi_i =  \begin{bmatrix}\xi^1_i \\ \xi^2_i \\ \vdots \\ \xi^\ell_i\end{bmatrix}
     = \begin{bmatrix}W_i^0&\bm{0}_{n_1\times n_1}&\dots&\bm{0}_{n_1\times n_{l-1}}\\
     \bm{0}_{n_2\times n_2}&\ddots&\ddots&\vdots\\
     \vdots&\ddots&\ddots&\bm{0}_{n_{\ell-1}\times n_{\ell-2}}\\
     \bm{0}_{n_l\times n_{x}}& \hdots & \bm{0}_{n_1\times n_{l-2}} &W_i^{\ell-1}\end{bmatrix}
     \begin{bmatrix}x^0_i\\ x^1_i\\\vdots\\ x^{\ell-1}_i\end{bmatrix}
     +\begin{bmatrix}b^0_i \\ b^1_i\\\vdots \\ b^{\ell-1}_i\end{bmatrix} 
     \in \mathbb{R}^M.
\end{align}
These vectors allow the $i^\text{th}$ NN of \eqref{nn1} be written in the more compact form  
 \begin{subequations}\label{full}\begin{align}
\xi^{k+1}_i  & = W_i^kx_i + b_i^k, \quad k= 0, 1, \dots, \ell-1,
\\
x^{k+1}_i  & = \phi(\xi^{k+1}_i), \quad k= 0, 1, \dots, \ell-1,
\\
f_i(x_i)  & = W^\ell_i\phi(\xi_i^\ell) + b_i^\ell.
\end{align}\end{subequations}
The vectors
\begin{align}\label{mu}
\mu = \begin{bmatrix} {\xi_1}^T, &  {\xi_2}^T, & {\phi(\xi_1)}^T, & {\phi(\xi_2)}^T, & 1 \end{bmatrix}^T, \eta =  \begin{bmatrix} {x_1}^T, &  {x_2}^T, & {\phi(\xi_1)}^T, & {\phi(\xi_2)}^T, & 1 \end{bmatrix}^T,
\end{align}
will also be used throughout the paper to build the various inequalities. These two vectors are linearly related through a matrix $E\in \mathbb{R}^{2(n_x+M)+1 \times 4M+1}$ by $\eta = E \mu.$

\subsection{The neural network similarity problem: Bounding the worst-case error between two neural networks} \label{sec:problemstatement}

The neural network \emph{similarity problem} is now introduced to bound the worst-case error between the outputs of two neural networks for all input data in pre-defined sets (e.g. hypercubes). Solutions to this problem indicate how similar two NNs are to each other. 

\begin{problem} \label{prob:general}
Given the  two neural networks $f_1(x_1):\mathcal{X} \to \mathcal{F}$ and $f(x_2):\mathcal{X} \to \mathcal{F}$ from \eqref{nn1},   minimise the worst-case output error bound
% \begin{align}\label{bound_def}
% ||f_1(x_1)-f_2(x_2)||_2^2 \leq \nu^T\Gamma_f\nu,
% \quad \forall x_1, x_2 \in \mathcal{X},
% \end{align}
% where 
% \begin{align} \label{bound_arg_def}
% \nu = \begin{bmatrix}{x_1}^T, &  {x_2}^T, & {x_1}^T - {x_2}^T, & 1 \end{bmatrix}^T, \quad \Gamma_f = blkdiag\{\begin{bmatrix}
% \mathbb{I}_{n_{x_1}}\gamma_1, & \mathbb{I}_{n_{x_2}}\gamma_2, & \mathbb{I}_{n_{x_1}\times n_{x_2}}\gamma_x, & \gamma
% \end{bmatrix}\}
% \end{align}
\begin{align}\label{bound_def}
||f_1(x_1)-f_2(x_2)||_2^2 \leq \gamma+ \gamma_{x_1}\|x_1\|_2^2+ \gamma_{x_2}\|x_2\|_2^2+ \gamma_{x}\|x_1-x_2\|_2^2,
\quad \forall x_1, x_2 \in \mathcal{X},
\end{align}
subject to $(\gamma_x,\,\gamma_{x_1},\,\gamma_{x_2},\gamma)\in \mathbb{R}_+$.

\end{problem}

The above problem establishes the general framework to address the specific problem of interest to this paper- that of bounding the worst-case error of a pruned or quantised neural network. To specialise Problem \ref{prob:general} to the computation of this NN quantisation bound, the second neural network in Problem \ref{prob:general} set to be the quantised one, as in 
\begin{align}
    W_2^k = q(W_1^k), \quad b_2^k = q(b_1^k), \quad  x_2 = q(x_1),
\end{align}
where $q(\cdot)$ is the quantisation function defined in Section \ref{sec:quant}. Likewise, upon considering pruned networks, Problem \ref{prob:general} is adapted by setting the second NN to be the pruned one, as in $W_2^k = p(W_1^k)$ where   $p(\cdot)$ is a pruning operator, see Section \ref{sec:Pruning}. 

%Computing tractable solutions to the neural network similarity problem (i.e., Problem. \ref{prob:general}), and applying this problem to computing robust bounds for quantised and pruned neural networks, is the main focus of this paper. The remaining sections would first define mathematical representations for the quantisation and pruning operation, then focus on derivation of the solutions.

% This problem the ... problem and solving in a tractable manner is the main focus of this paper, as it directly allows us to compute the error from the computationally efficient implementation of the NN. In general, providing solutions to this problem is somewhat intractable, primarily because of the nonlinear activation functions $phi(cdot)$, and the fact that the bounds must hold for all inputs in the set $mathcal{X}$. To overcome this issue, quadratics constraints are used to over-approximate these sets, which enable solutions to Problem... to be computed from the solution of a SDP. 

%The main tools used for reducing this problem to a convex SDP-feasibility problem are quadratic constraint abstractions of 1. the input set $\mathcal{X}$, 2. non-linear properties of the activation functions, 3. the comparison error in equation (\ref{bound_def}). The goal is to combine them as a linear matrix inequality constraint to a convex optimization problem, which can be solved to produce the parameters $\gamma_x$ and $\gamma$ to the bound generated.

\section{Quantised neural networks} \label{sec:Quantise}

This section outlines a mathematical representation of quantisation when applied to neural network elements. With this representation, network quantisation can be incorporated into the neural network similarity problem of Problem \ref{prob:general}, to bound the effect of additionally quantising the network input and neuron parameters.

\subsection{Fixed-point arithmetic}

It will be assumed that the quantised neural network’s weights, biases and data are stored using fixed-point arithmetic, detailed in \cite{yates2009fixed, 8109980}. The $\langle \texttt{IB,FB} \rangle$ format for fixed-point representations adopted in \cite{DBLP:journals/corr/GuptaAGN15,210171} is used, whereby \texttt{IB} and \texttt{FB} respectively denote the number of bits allocated for the integer and fractional parts of the number. As an example, the number 31.4592 is represented as
\vspace{-0.3cm}
\begin{figure}[H]
\centering
    \includegraphics[width=0.4\textwidth]{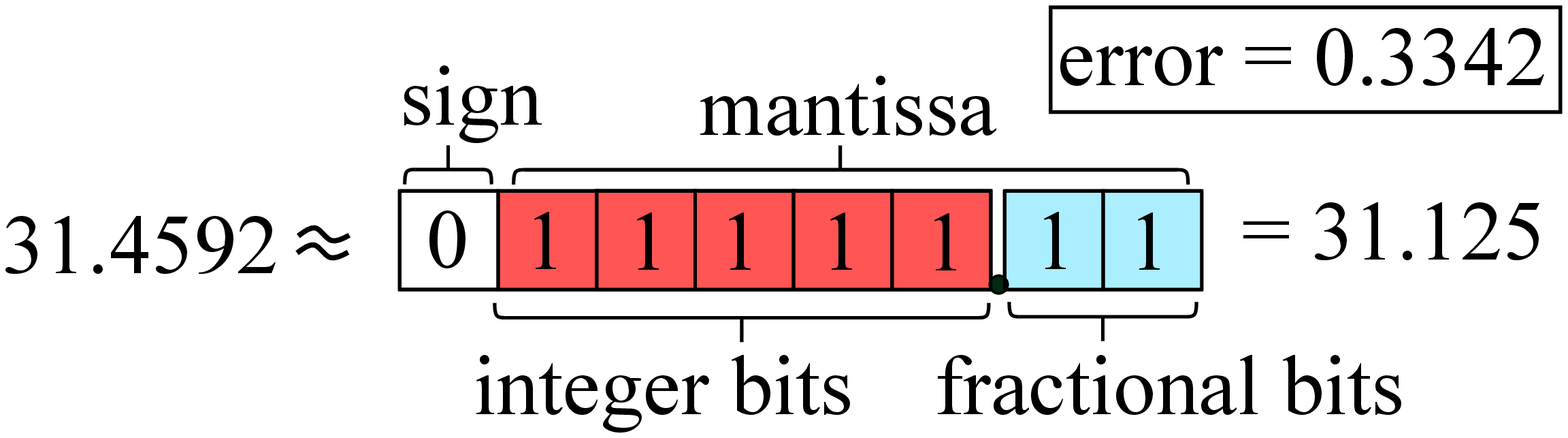}
\end{figure}
\vspace{-0.5cm}
In this work, saturation of the representation is neglected. This is an important issue but could be avoided with the introduction of ``temporary fixed-point registers with enough number of bits'' (\cite{DBLP:journals/corr/GuptaAGN15}).

% \begin{figure}[htbp]
% \floatconts
% {fig:q(x)}% label for whole figure
% {\caption{Quantisation function- FIX THIS.}}% caption for whole figure
%     {%
%     \subfigure[Quantised signal]{
%     \label{fig:q(x)}
%     \includegraphics[width = 0.4\linewidth]{Figures/quantisation.jpg}
%     }\qquad % space out the images a bit
%     \subfigure[Quantisation step properties]{
%     \label{fig:qstep}
%     \includegraphics[width = 0.4\linewidth]{Figures/quantisation quadratic constraint.jpg}
%   }
%   }
% \end{figure}

% \begin{figure}
% \floatconts
% {fig:q(x)}% label for whole figure
% {\caption{Illustrations of the 8-bit fixed-point representation of 31.4592 and the quantisation function used to represent the number discretisation in the analysis.}}% caption for whole figure
%     {%
%     \subfigure[{Illustration of an 8-bit fixed-point representation of 31.4592 with 5 integer bits and 2 fractional bits.}]{
%     \label{fig:FixedPointArith}\centering
%     \includegraphics[width = 0.4\linewidth]{Figures/fixed_point} 
%     }\qquad % space out the images a bit
%     \subfigure[{Quantisation function with its upper and lower linear bounds.}]{
%     \includegraphics[width = 0.4\linewidth]{Figures/quantisation_RD_mod}
%     \label{fig:quantisation}
%   }
%   }
% \end{figure}

\subsection{Quantisation function} \label{sec:quant}

The piece-wise constant quantisation function $q(\cdot)$ acting element-wise on vector and matrix arguments can be used to capture the fixed-point representation of the quantised neural network's weights, biases and data. % as illustrated in Figure. \ref{fig:q(x)}.  

Given a real number $s \in \mathbb{R}$, the quantisation function is defined as
\begin{align} \label{equ:q(x)}
    q(s) \triangleq \frac{\text{sgn}(s)\floor*{|s| \cdot \Delta}}{\Delta},    \Delta = 2^{- \texttt{FB}}, \texttt{FB} \in \mathbb{Z}^{+},
\end{align}
where $\floor*{\cdot}$ is the ``floor'' operator and $\Delta$ is the binary quantisation step, with the word length \texttt{FB} setting the numerical precision of the quantisation. For the robustness analysis of the following sections, this function will be upper and lower bounded by 
\begin{figure}[H]
\centering
    \includegraphics[width=0.25\textwidth]{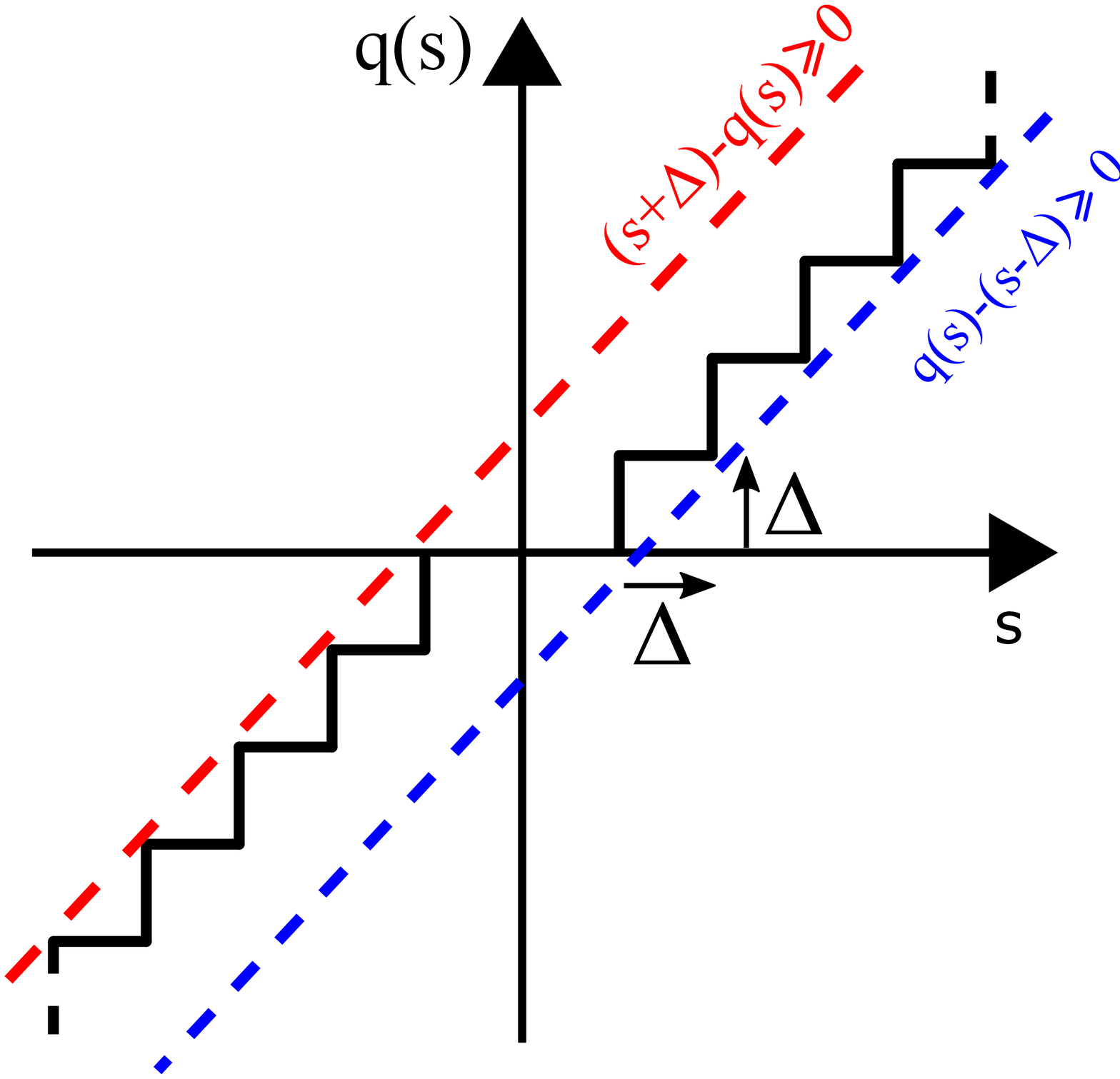}
\end{figure}

Decreasing \texttt{FB} causes a coarse fixed-point representation, resulting in information being lost because of the quantisation and the neural network's output being corrupted, inducing an error. Given a binary number ($BN$) represented in fixed-point format, the original decimal format ($DN$) can be recovered from 
\begin{equation}
    DN = \sum_{i=1}^{\texttt{IB}-1} BN_{\langle I \rangle_i} * 2^{i} + \sum_{j=1}^{\texttt{FB}-1} BN_{\langle F \rangle_j} * 2^{-j},
\end{equation}\label{equ:bi2de}
with $\langle I \rangle$ and $\langle F \rangle$ denoting the integer and fractional parts of $BN$.

% \begin{wrapfigure}{r}{0.3\textwidth}
%   \begin{center}
%     \includegraphics[width=0.3\textwidth]{Figures/quantisation_RD_mod}
%   \end{center}
%   \caption{Quantisation function with its upper and lower linear bounds.}
%   \label{fig:quantisation}
% \end{wrapfigure}

Note in this work the floor rounding scheme was used, i.e. $q(\cdot)$ maps an input value to the largest integer multiple of $\Delta$ smaller than the input value. Other rounding schemes exist and may be adapted into the proposed framework, e.g. round-to-nearest (\cite{DBLP:journals/corr/GuptaAGN15}) and various piece-wise-linear and stochastic rounding methods.

As ReLU activation functions were adopted for the numerical results of this paper, the effect of additionally quantising the activation functions was ignored. However, this effect could be incorporated into the analysis as long as the quantised activation functions satisfy the quadratic constraints. 

\section{Pruned neural networks} \label{sec:Pruning}

Neural network pruning involves setting the weights or neurons considered to be the least important to zero, with the purpose of making the network structure sparse and so improve its computational and memory efficiency (\cite{hooker2020compressed}). A state-of-the-art pruning method is presented in \cite{han2015learning} where a three-step framework is applied to train, prune and refine neural network parameters. Using this framework both AlexNet and VGG-16 have been successfully and effectively downsized with no loss in accuracy (\cite{han2015learning}). In this paper, the operator $p(\cdot)$ denotes the pruning action implemented with a magnitude based approach, although other pruning approaches may also be implemented. By setting the second NN in \eqref{nn1} to be the pruned version, Problem \ref{prob:general} can be solved to compute the pruning error bounds.

% Similarly to the representation in Section \ref{sec:Quantise}, pruning can also be represented in the SDP formulation to produce robust error bounds for pruned neural networks. Robust error bounds for network pruning can also be established using this framework. %Abstracting the non-linear properties of the network and translating them into constraints for Problem. \ref{prob:general} are discussed in sections that follow.

%-------------------------------------------------------------------------------------
\section{Abstracting the nonlinearities and sets using quadratic constraints} \label{sec:abstraction}

In general, solving Problem \ref{prob:general} is intractable because of: (i) the potentially non-convex mapping of each NN with nonlinear activation functions, and  (ii) the bounds are required to hold robustly for all inputs in the pre-defined set $(x_1,\,x_2) \in \mathcal{X}$ with $\mathcal{X}$ allowed to be infinite dimensional. Using quadratic constraints, this paper is able to provide an upper bound for the error from neural network pruning and quantisation by producing an algebraic, convex outer approximation of the various sets and nonlinearities of the problem. This allows solutions to Problem \ref{prob:general} to be computed via a SDP, with the conservatism of the approach coming from the fact that it only implicitly characterises the various nonlinearities and sets of the analysis. The remaining part of this section details the quadratic constraints for the input set, the nonlinear activation functions, the quantisation and the error bound \eqref{equ:Mgamma}. 
%-------------------------------------------------------------------------------------
\subsection{Abstraction of the input set}
 The various quadratic constraints satisfied when the NN inputs are contained to the set $(x_1,\, x_2)\in \mathcal{X}$ are detailed in Appendix 3. When the inputs are constrained to this set, there exists a matrix $P_\mathcal{X}\in \mathbb{R}^{2(n_x+M)+1\times 2(n_x+M)+1}$  built from the matrix variables (the lambdas of Definition \ref{def:input_hypercube} in Appendix 3) such that
\begin{align}\label{equ:Mx}
M_{\mathcal{X}}(P_\mathcal{X})
= \eta^TP_\mathcal{X}\eta \geq 0, \quad \forall (x_1,\,x_2) \in \mathcal{X}.
\end{align}

%-------------------------------------------------------------------------------------

\subsection{Abstraction of the activation functions}
Dealing with the nonlinear activation functions is the main source of difficulty in generating robust solutions to Problem \ref{prob:general}. Thankfully,  most of the commonly used NN activation functions (e.g. those mentioned in Table \ref{tab:activationfunc} of Appendix 3 have properties that allow them to be characterised by quadratic constraints. The mathematical definitions of these properties can be found in Definition. \ref{def:nonlinearfunctional} of Appendix 4 and the associated quadratic constraints are given in Lemma \ref{lem:ReLUQCs} of Appendix 4.

Using the compact NN representation of  \eqref{full}, the various quadratic constraints satisfied by the activation functions of Lemma \ref{lem:ReLUQCs} in Appendix 4 can be collected together and represented as a single inequality
\begin{align}\label{equ:Mlambda}
M_{\phi}(\Lambda) = \mu^T \Lambda \mu=  \eta^T \left(E^T\Lambda E\right) \eta\geq  0, \quad \forall (x_1,\,x_2) \in \mathcal{X}.
\end{align}
Here, the matrix
\begin{align}
\Lambda = \begin{bmatrix} \bm{0}_{N \times N}&  \bm{0}_{N \times N} &  \Lambda_{13} &  \Lambda_{14} &  \Lambda_{15} 
\\
 \bm{0}_{N \times N} &   \bm{0}_{N \times N}&  \Lambda_{23} &  \Lambda_{24} &  \Lambda_{25} 
\\
{\Lambda_{13}}^T &  {\Lambda_{23}}^T &  \Lambda_{33} &  \Lambda_{34} &  \Lambda_{35} 
\\
{\Lambda_{14}}^T &  {\Lambda_{24}}^T &  {\Lambda_{34}}^T &  \Lambda_{44} &  \Lambda_{45} 
\\
{\Lambda_{15}}^T & {\Lambda_{25}}^T &  {\Lambda_{35}}^T & {\Lambda_{45}}^T &  \Lambda_{55} 
\end{bmatrix}
\end{align}
 is composed of elements $\Lambda_{i,j}, \{i, j\} \in \{1,2,3,4,5\}$ being linear combinations of the scaling variables of  the various quadratic constraints of $\phi(\cdot)$, e.g. the lambdas of Lemma \ref{lem:ReLUQCs} in Appendix 4.
 
Note that information on the quantisation and pruning of the NN parameters (weights and biases), or any other alteration in value or structure to the parameters in general, is introduced a priori into the SDP problems' constraints via the abstractions just defined, and therefore, does not require extra QCs. %This is one of the major advantages and novelties to our framework.

%-------------------------------------------------------------------------------------

\subsection{Abstraction of the quantisation function}

Besides the activation functions $\phi(\cdot)$, for quantised NNs the quantisation function $q(x_i)$ from \eqref{equ:q(x)} acting on the input data is another nonlinearity that will have to to be accounted for in the analysis if error bounds are to be obtained. Thankfully, this nonlinear function also satisfies certain quadratic constraints as defined in Lemma \ref{lem:quant} of Appendix 5. These quadratic constraints  can be combined into the single inequality
\begin{align}\label{equ:Mq}
   M_q(P_q) = \eta^T P_q\,\eta \geq 0, \quad \forall (x_1,\,x_2) \in \mathcal{X},
\end{align}
for some matrix variable $P_q$ built from the various lambdas of Lemma \ref{lem:quant} of Appendix 5. 

%-------------------------------------------------------------------------------------

\subsection{Abstraction of the error bound}
 The bound for the worst-case error between the two neural networks
\begin{align}\label{equ:error_2}
||f_1(x_1)-f_2(x_2)||_2^2 -(\gamma+ \gamma_{x_1}\|x_1\|_2^2+ \gamma_{x_2}\|x_2\|_2^2+ \gamma_{x}\|x_1-x_2\|_2^2)\leq 0, \quad \forall (x_1, x_2) \in \mathcal{X}
\end{align}
can also be expressed as a quadratic
\begin{align}\label{equ:Mgamma}
M_{\|f-g\|_2}(\Gamma) = \eta^T V^T \Gamma(\gamma_{x_1}\,\gamma_{x_2},\,\gamma_x,\,\gamma)\, V \eta \leq 0, \quad \forall (x_1,\,x_2) \in \mathcal{X},
\end{align}
where 
\begin{align}
\Gamma = - \begin{bmatrix} (\gamma_{x_1} + \gamma_x) \mathbb{I}_{n_{x_1}} &  -\gamma_x \mathbb{I}_{n_{x_1}} &  \bm{0}_{n_{x_1} \times n_{f_1}} &  \bm{0}_{n_{x_1} \times n_{f_2}} &  \bm{0}_{n_{x_1} }
\\
-\gamma_x\mathbb{I}_{n_{x_2}}  &   (\gamma_{x_2} + \gamma_x)\mathbb{I}_{n_{x_2}} &  \bm{0}_{n_{x_2} \times n_{f_1}} &  \bm{0}_{n_{x_2} \times n_{f_2}} &  \bm{0}_{n_{x_2}} 
\\
\bm{0}_{n_{f_1} \times n_{x_1}}  &   \bm{0}_{n_{f_1} \times n_{x_2}}  &  \mathbb{I}_{n_{f_1}}  &  -\mathbb{I}_{n_{f_1}}  &  \bm{0}_{n_{f_1}}
\\
\bm{0}_{n_{f_2}} &  \bm{0}_{n_{f_2} \times n_{x_2}} &  - \mathbb{I}_{n_{f_2}}  &  \mathbb{I}_{n_{f_2} } &  \bm{0}_{n_{f_2} }
\\
{\bm{0}_{n_{x_1}}}^T  & {\bm{0}_{ n_{x_2}}}^T  &  {\bm{0}_{n_{f_1}}}^T & {\bm{0}_{1 \times n_{f_2}}}^T &  \gamma
\end{bmatrix},
\end{align}
and
\begin{align}
    V = \begin{bmatrix} \mathbb{I}_{n_{x_1}} &\bm{0}_{n_{x_1}\times n_{x_2}}  & \bm{0}_{n_{x_1}\times M} & \bm{0}_{n_{x_1}\times M} & \bm{0}_{n_{x_1}}  \\
    \bm{0}_{n_{x_2}\times n_{x_1}} &  \mathbb{I}_{n_{x_2}} &\bm{0}_{n_{x_2}\times M} & \bm{0}_{n_{x_2}\times M} & \bm{0}_{n_{x_2}}  \\ 
    \bm{0}_{n_{f_1}\times n_{x_1}}  & \bm{0}_{n_{f_1}\times n_{x_2}} & \begin{bmatrix}\bm{0}_{n_{f_1}\times M-n_{\ell_1}} & W^{\ell}_1\end{bmatrix} & \bm{0}_{n_{f_1}\times M}  & b^{\ell}_1 \\ 
    \bm{0}_{n_{f_2}\times n_{x_1}} & \bm{0}_{n_{f_2}\times n_{x_2}} & \bm{0}_{n_{f_2}\times M} &  \begin{bmatrix}\bm{0}_{n_{f_2}\times M-n_{\ell_2}} & W^{\ell}_2\end{bmatrix}  &  b^{\ell}_2
    \\ {\bm{0}_{ n_{x_1}}}^T & {\bm{0}_{ n_{x_2}}}^T & {\bm{0}_{M}}^T & {\bm{0}_{M}}^T & 1
    \end{bmatrix}.
\end{align}
Here,  for a fair comparison, $n_{x_1} = n_{x_2}$ and $n_{f_1} = n_{f_2}$, as in both NNs have the same number of inputs and outputs. %With respect to the metric defined in \eqref{bound_def} of Problem \ref{prob:general}, \eqref{equ:error_2} uses the squared 2-norm ($p = q = 2$). The main advantage of using the 2-norm in this context is that the error bounds can be computed from the solution of a SDP, although that the $\infty$-norm is normally preferred for machine learning applications. However, the computational advantages of the quadratic 2-norm justified its use here.   

%-----------------------------------------------------------------------------------------------------------------------------------------------
\section{Neural network error bounds as a semidefinite programme}\label{sec:sdp}
With the QCs defined for each feature of the problem, upper bounds to \eqref{bound_def} in Problem \ref{prob:general} can be computed from the SDP of Theorem \ref{thm:SDP}. Worst-case error bounds for pruned neural networks can be obtained directly upon setting the second network of the problem to be the pruned one. The specialisation to quantised neural networks has the quantisation quadratic constraints of \eqref{equ:Mq} included within the linear matrix inequality \eqref{LMI}.

\begin{theorem}\label{thm:SDP}
Consider the two neural networks defined in \eqref{nn1}  satisfying the quadratic constraints of \eqref{equ:Mx}, \eqref{equ:Mlambda} and \eqref{equ:Mq}. Set the weights $(w_{x_1},\, w_{x_2},\,w_x,\,w_\text{aff}) \in \mathbb{R}_+$ scaling the relative importance of the quadratic and affine terms in the error bound \eqref{equ:error_2}. If there exists a solution to 
\begin{subequations}\label{thm_cond}\begin{align}
 \text{min} \quad&  w_{x_1}\gamma_{x_1} + w_{x_2}\gamma_{x_2} + w_x\gamma_x + w_\text{affine}\gamma,\\
 \text{subject to} \quad& P_\mathcal{X}(\cdot) + P_q(\cdot)+ E^T\Lambda(\cdot) E + \Gamma(\cdot) \prec 0, \label{LMI}\\ \text{          }\quad& \gamma_{x_1} \geqslant 0,~\gamma_{x_2} \geqslant 0,~\gamma_x \geqslant 0, \,\gamma \geqslant 0.
\end{align}\end{subequations}
with the matrix variables $P_\mathcal{X}(\cdot), \,P_q(\cdot),\,\Lambda(\cdot)$ and $ \Gamma(\cdot)$ defined in \eqref{equ:Mx}, \eqref{equ:Mq}, \eqref{equ:Mlambda} and \eqref{equ:Mgamma}, then the error between the two neural networks is bounded from above by 
\begin{align}\label{bound_thm}
\|f_1(x_1)-f_2(x_2)\|_2^2 \leq  \gamma+ \gamma_{x_1}\|x_1\|_2^2+ \gamma_{x_2}\|x_2\|_2^2+ \gamma_{x}\|x_1-x_2\|_2^2, \quad (x_1,\,x_2) \in \mathcal{X}.
\end{align}\label{theorem:SDP}
\end{theorem} 
\begin{proof}
See Appendix 6.
\end{proof}

%------------------------------------------------------------------------------------------------------------------------------------------------
\section{Numerical examples and discussion}

In this section, the effectiveness of the obtained bounds was explored with several numerical examples. The first subsection presents examples for bounding the error between two generic neural networks, and the second is focused on quantised and pruned networks.

% \subsection*{Computational Setup}
In all the examples, the input was constrained to the hyper-cube $(x_1,\,x_2) \in \mathcal{X} = \{x_1, x_2 \in \mathbb{R}^{n_x} \mid |x_1|, |x_2| \leq \overline{x} \}$, where  $\overline{x} = 1$ and the ReLU was chosen as the activation function. For all examples, the SDP of Theorem \ref{theorem:SDP} was solved with the MOSEK \citep{mosek} solver implemented through the YALMIP \citep{Lofberg2004} interface of MATLAB.  The code and the various models used in the examples can be found at:  \url{https://github.com/ElrondL/Robust-Quantisation-Bounds}.

% Other open-source or commercial optimizers such as SEDUMI \citep{doi:10.1080/10556789908805766} and SDPT3 \citep{doi:10.1080/10556789908805762} are also suitable to the framework.
%------------------------------------------------------------------------------------------------------------------------------------------------
\subsection{Neural network similarity bounds}

For the first experiment, the performance of the framework defined in Theorem \ref{theorem:SDP} was checked by comparing randomly generated neural networks with $n_x = 1$ as the input dimension, $n_f = 1$ as the output dimension and $\ell = 1, \,2, \,3, \,4$ hidden layers, each with $n_k = 10$ neurons. The two neural networks were fed with randomly generated signals $\{x_1, x_2\} \in \mathcal{X}$. For each NN layer dimension $(\ell)$, 100 random neural networks were generated and the average $\gamma_{x_1}$, $\gamma_{x_2}$, $\gamma_x$, $\gamma$, mean bound tightness, maximum bound tightness, minimum bound tightness, and runtime were recorded. Tightness of the bound was defined as the natural log of the error between  the output difference norm and the over approximation bound 
\vspace{-0.0cm}
\begin{align} \label{equ:boundT}
    T = \ln((f(x_1)-f(x_2))^2) - \ln(\gamma_{x_1} x_1^2 + \gamma_{x_2} x_2^2 + \gamma_x (x_1-x_2)^2 + \gamma),
\end{align}
with the results given in Table \ref{tab:NNcompData} of Appendix 7.
\vspace{-0.2cm}
\begin{figure}[H]
\floatconts
  {fig:NNCompFig}
  {\caption{Example neural network error bound with $\ell = 2$.}}
  {%
    \subfigure[The error (coloured) and bound (grey) surfaces.]{\label{fig:NNCompigBE}%
      \includegraphics[width=0.43\linewidth]{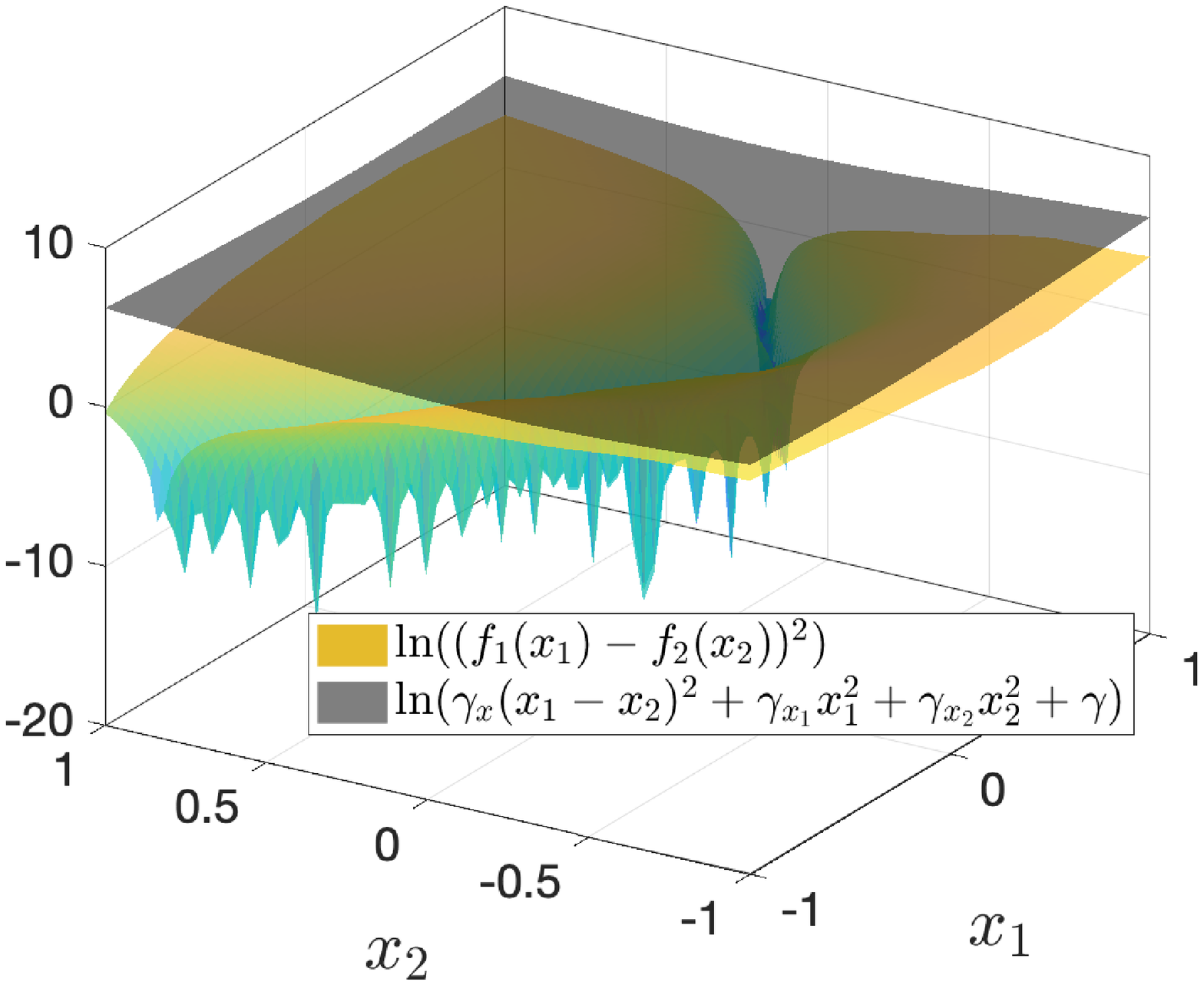}}%
    \qquad
    \subfigure[The tightness surface $T$ of the bound and error surfaces in Fig.\ref{fig:NNCompFig}(a)]{\label{fig:NNCompFigT}%
      \includegraphics[width=0.43\linewidth]{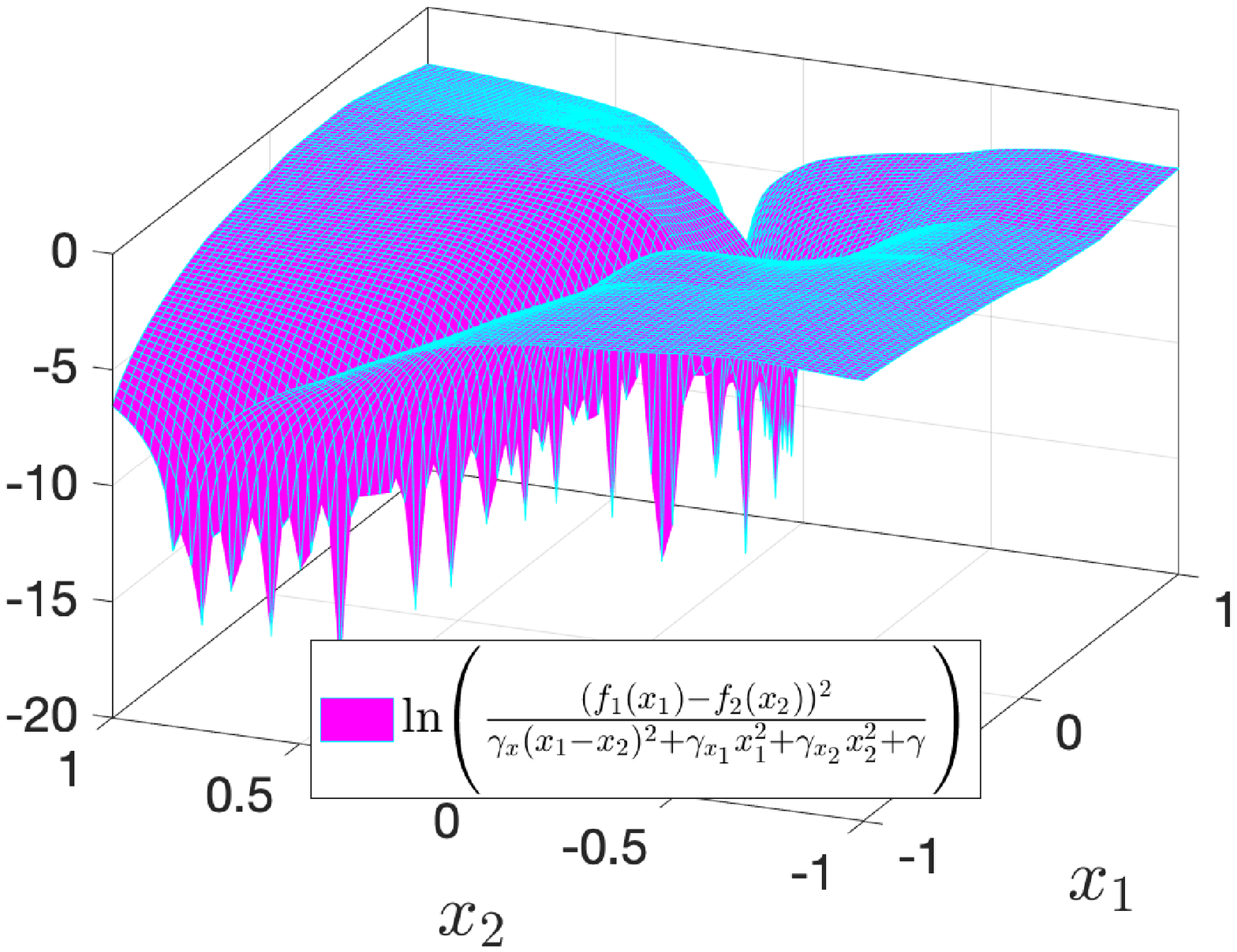}}
 \vspace{-0.7cm} }
\end{figure}
A bounding surface holding robustly for all inputs generated was obtained for each test case, an example of which can be found in Figure \ref{fig:NNCompFig}. When $n_x = 1$ and $n_f = 1$, the bound follows the error surface with $T \leq 2$ for the majority of the input space. The singularity of the log function at zero means that points in the input space generating small errors between the two neural networks are associated with the deep valleys in the figure.  The average distance between the bound and the error surface became larger with an increasing number of neurons and for multi-dimensional input and output vectors, as reflected in Table \ref{tab:NNcompData}'s data.

%------------------------------------------------------------------------------------------------------------------------------------------------
\subsection{Quantisation error bound}
Error bounds were also computed for quantised NNs where $W_2 = q(W_1)$, $b_2 = q(b_1)$ and $x_2 = q(x_1)$. Here, the output dimension was $n_f = 1$,  the input dimension was $n_x = 1$ with the input set for $x_1$  consisting of 100 evenly spaced values in $[-1, 1]$. The quantisation step was $\Delta = 2^{-2}$, as in two bits were used to store each number. Figure \ref{fig:QCompFig} in Appendix 7 contains example error bound curves for the various models considered and also an example in 3D (where the condition $x_2 = q(x_1)$ was relaxed, giving a looser bound as the input space was less restricted). % generated using Theorem \ref{theorem:SDP} and Definition \ref{def:QC4ReLU} only because Lemma \ref{lem:quant} imposes element-to-element relationships between the two input vectors). %QC posed in Lemma. \ref{lem:quant} typically reduces conservatism of the bound. 

Bounds were then computed for 100 randomly generated NNs with $\ell= 4$ hidden layers, with the corresponding results shown in Table \ref{tab:QcompData} of Appendix 7. Bound tightness was defined as the error between the original and quantised networks' output difference norm and the over-approximation bound of \eqref{equ:boundT}.

A bounded curve holding for all inputs was generated for each test case, with the step patterns of the error curves following from the quantisation. It was also observed that the bound loosened with an increasing number of neurons, as seen in Table \ref{tab:QcompData} of Appendix 7.

\subsection*{Worst-case quantisation error bound}
The worst-case quantisation error and bound values were also recorded and averaged for 100 random neural networks of $\ell= 2$ layers for quantisation levels $\Delta = 2^i, i \in \{1,2,3,4,5\}$. Note that the bound and error values come in pairs, so there exists a bound value corresponding with the maximum error as well as an error value corresponding to the maximum bound value. These results where then plotted in Figure \ref{fig:WorstCaseQ} of Appendix 7 where the maximum bound values tended to be close to the observed maximum error, however the gap between the corresponding maximum error and the bound tended to be quite large. It was also observed that both the error and the bound decreased at similar rates with an increasing level of quantisation. %This shows that bound surface generated is able to follow the error surface with consistency.

\subsection*{Pruned neural networks}
Figure \ref{fig:NNPruned} in Appendix 7 shows the results for bounding the error between a randomly-generated 20 neuron neural network and its pruned version (with $f_2(x_2)$ in Problem \ref{prob:general} being the pruned one), with the weights of the eight hidden neurons with smallest 2-norms set to zero (\cite{blalock}). Note, that by setting the second neural network to be identical to the first, the framework could also bound the behaviour of a neural network within its input space. Pruning that removes neurons rather than setting weights to zero could also be bounded with this framework. 

%------------------------------------------------------------------------------------------------------------------------------------------------
\section*{Conclusions}

A method to generate uppers bounds for the worst-case error induced by either quantising or pruning a trained neural network was introduced. The bounds hold robustly for all inputs in pre-specified sets, account for nonlinear activation functions and are generated from the solution of a semi-definite programme. Several numerical examples graphically illustrated the tightness of the bounds. Future work will explore reducing the conservatism of the bounds, improving the scalability of the semi-definite programming solvers and applying the method to other problems of interest where the similarity of two neural networks is to be evaluated.

% Acknowledgments---Will not appear in anonymized version
\acks{Ross Drummond would like to thank the Royal Academy of Engineering for funding this research through a UKIC Fellowship as well as the Nextrode project of the Faraday Institution. Jiaqi Li would like to thank the Department of Engineering, University of Oxford for funding through the UROP scheme.}

\bibliography{bibliog} 

%-------------------------------------------------------------

\section*{Appendices}

\subsection*{Appendix 1: Illustration of the neural network similarity problem}
\begin{figure}[H]
\centering
\includegraphics[width = 0.5\linewidth]{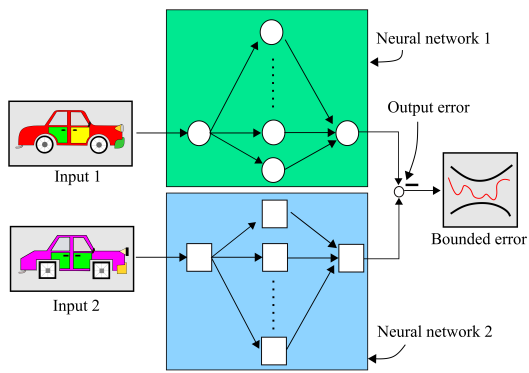}
\caption{Illustration of the similarity problem which bounds the worst case error between the outputs of two neural networks for all inputs in some set $(x_1,\,x_2) \in \mathcal{X}$. By setting the second neural network in this set-up to be the quantised or pruned one, the desired error bounds can be computed. }
\label{fig:colab}
\end{figure}

\subsection*{Appendix 2: Quadratic constraints overview}

As suggested by their name, a quadratic constraint is simply a non-negative quadratic algebraic expression capturing dependencies between its arguments. For instance, the following is satisfied by the ReLU nonlinearity
\begin{align}
    p_i(s,\sigma) = (s-\text{ReLU}(s))\sigma \geq 0, \quad \forall (s,\sigma) \in \mathbb{R}.
\end{align}
Now, consider the case where we have $N$ quadratic constraints $p_i(\cdot) \geq 0$ for $i = 1,\, \dots, \, N$ mapping $\mathbb{R}^n \to \mathbb{R}_+$ but want to show the negativity of another  $p_0(s,\sigma) < 0$. This is obviously true if there exists $\lambda_i \geq 0, \, i = 1, \, \dots, \, N$ such that 
\begin{align} \label{s-proc}
p_0(\cdot) + \sum^N_{i=1}\lambda_i p_i(\cdot) < 0.
\end{align}
The above is a version of the S-procedure \citep{ulf} but simplified to static systems- and it is this feature that is utilised in this work to compute the bonds. The benefit of working with \eqref{s-proc} is that it is a linear matrix inequality (LMI) \citep{ulf}, a class of well-understood convex problems. For the particular problem considered in the paper, the quantity to be checked $p_0(s,\sigma) < 0$  relates to the error bound whilst the specified non-negative quadratic constraints characterise properties of the activation functions, inputs sets and quantisation.

Quadratic constraints have a long history in control theory to encode information about unknown nonlinear disturbances (e.g. control saturation \citep{tarbouriech} and quantisation \citep{lam}) within a robustness analysis. For the specific problem considered in this paper, they bring the benefits of: (i) a common framework for describing the many features of the problem, including the various nonlinearities and sets, (ii) verifying a quadratic constraint can be cast as a LMI- a class of well-studied convex optimisation problems. The two main drawbacks are: (i) they can lead to conservative bounds as the bounds must hold for all terms satisfying the quadratic algebraic expressions, and (ii) even though a LMI is a convex problem, its complexity scales poorly ($\mathcal{O}(n^6)$ in runtime and $\mathcal{O}(n^4)$ in memory requirements, where $n$ is the number of neurons in the network \citep{dathathri}). However, the ability to compute worst-case bounds justified its consideration here.

\subsection*{Appendix 3: LMI quadratic constraints for input hyper-rectangle}\label{app:input}

Each neural network's input ($x_1,\, x_2)\in \mathcal{X}$ is assumed to lie within hyper-rectangles.
\begin{definition}
For the $i^\text{th}$ neural network with $i \in \{1,2\}$, define the hyper-rectangle containing the input $x_i$ as $\mathcal{X}_{i}= \{x_i  : \underline{x}_i \leq x_i \leq \overline{x}_i \}$. Within these hyper-rectangles, the combined input data for both NNs is constrained to $(x_1,\,x_2) \in \mathcal{X} $, with $\mathcal{X}$ potentially capturing cross-conditions relating the two inputs (for example the input of the second NN being the quantised version of the first, $x_2 = q(x_1))$.
\end{definition}

The hyper-rectangles $\mathcal{X}_i, \, i \in \{1,\,2\}$ are represented by the following quadratic constraints.

\begin{definition}\label{def:input_hypercube}
For $i \in \{1,\,2\}$, if $x_i \in \mathcal{X}_i$ then
\begin{subequations} \label{qc:input}
$
\begin{bmatrix}{x_i}^T &  1 \end{bmatrix}P_{\mathcal{X}_{i}}\begin{bmatrix}x_i \\ 1\end{bmatrix} \geq  0, 
$
where
\begin{align}
P_{\mathcal{X}_{i}} = \begin{bmatrix}  -\lambda_{x_{\infty}}& \frac{\lambda_{x_{\infty}}}{2}\left(\underline{x}_i + \overline{x}_i\right)  
\\ 
\left(\underline{x}_i + \overline{x}_i\right)\frac{\lambda_{x_{\infty}}}{2} &  -\underline{x}_i^T\lambda_{x_\infty}\overline{x}_i \end{bmatrix}, \quad \lambda_{x_{\infty}}\in \mathbb{D}^{n_x}_{+}.
\end{align}\end{subequations}
In addition, if $\overline{x}_1 = \overline{x}_2= \overline{x}$ and $\underline{x}_1 = \underline{x}_2= \underline{x}$, then the two inputs $(x_1,x_2)$  also the satisfy joint quadratic constraints of the form
\begin{subequations}  \begin{align} \label{qc:input_joint}
\begin{bmatrix}(2\overline{x} - (x_1 + x_2)) \\ (\overline{x}-\underline{x}) - (x_2-x_1) \end{bmatrix}^T P_{\mathcal{X}_{joint}}\begin{bmatrix}(x_1 + x_2)-2\underline{x} \\ (x_2-x_1)+\overline{x}-\underline{x}\end{bmatrix} \geq  0, 
\end{align}
where 
\begin{align}
P_{\mathcal{X}_{joint}} = \begin{bmatrix}  \lambda_{x_1+x_2}& 0  
\\ 
0 &\lambda_{x_1-x_2}\end{bmatrix}, \quad  (\lambda_{x_1+x_2}, \lambda_{x_1-x_2})\in \mathbb{D}^{n_x}_{+}.
\end{align}
\end{subequations}
Other joint conditions may also be satisfied by $(x_1,\,x_2) \in \mathcal{X}$, for instance when the input is quantised $x_2 = q(x_1).$

\end{definition}

\subsection*{Appendix 4: LMI quadratic constraints for the activation functions}
\begin{table} 
\centering
 \begin{tabular}{|l|c|c|c|c|c|c|}
 \hline
  $\phi(\cdot)$ property & Sigmoid-$\frac{1}{2}$ & tanh & ReLU  & ELU & Quantisation  & Saturation\\ \hline
  Sector bounded      &  \checkmark &  \checkmark &  \checkmark  & \checkmark & \checkmark  & \checkmark \\
  Slope restricted      &  \checkmark &  \checkmark & \ \checkmark & \checkmark & \checkmark  & \checkmark\\
  Bounded & \checkmark&  \checkmark &  $\times$  & $\times$  & $\times$  & \checkmark \\
  Positive & $\times$  & $\times$  & \checkmark  & $\times$ &  $\times$ &  $\times$   \\
  Positive complement & $\times$   & $\times$  & \checkmark  & $\times$  & $\times$ &  $\times$ \\
  Complementarity condition         & $\times$    & $\times$  &   \checkmark & $\times$  & $\times$ &  $\times$   \\
  Affine lower bound & $\times$& $\times$ & $\times$ & $\times$  & \checkmark   &  $\times$ \\
  Affine upper bound &$\times$ & $\times$& $\times$& $\times$  & \checkmark &  $\times$   \\
  \hline
 \end{tabular}
 \caption{Properties of commonly used activation functions, including the sigmoid, tanh, rectified linear unit ReLU and exponential linear unit (ELU). }
 \label{tab:activationfunc}
\end{table}

\begin{definition}[\cite{drummond2021reducedorder}] \label{def:nonlinearfunctional}
Consider a function $\phi(s): \mathcal{S} \subseteq \mathbb{R} \to \mathbb{R}$ satisfying $\phi(0) = 0$. This function is said to be sector bounded if 
\begin{subequations}\begin{align}
\frac{\phi(s)}{s} \in [0, \delta] \quad \forall s \in \mathcal{S},
\end{align}
slope restricted if 
\begin{align}
\frac{d\phi(s)}{ds} \in [\underline{\beta},\, \beta],  \quad \forall s \in \mathcal{S}, \quad \beta >0, 
\end{align}
and monotonic if $\underline{\beta} = 0$. If $\phi(s)$ is slope restricted, then it is also sector bounded. The nonlinearity $\phi(s)$ is bounded if
\begin{align}
\phi(s) \in [\underline{c}, \overline{c}],\quad \forall s \in \mathcal{S},
\end{align}
positive if
\begin{align}
\phi(s) \geq 0 , \quad \forall s \in \mathcal{S},
\end{align}
its complement is positive if
\begin{align}
\phi(s)-s \geq 0 , \quad \forall s \in \mathcal{S},
\end{align}
and it satisfies the complementarity condition if 
\begin{align}
(\phi(s)-s)\phi(s) = 0, \quad \forall s \in \mathcal{S}.
\end{align}\end{subequations}
\end{definition}

\begin{lemma} \label{lem:ReLUQCs}
Consider $\xi_1, \xi_2 \in  \mathbb{R}^{M} $, where $\xi_1, \xi_2$ are the arguments of the network activation functions $\phi(\cdot)$ defined in \eqref{xi_def}. If $\phi(\xi) : \mathbb{R}^{M}\to \mathbb{R}^{M}$ is a sector-bounded nonlinearity then
\begin{align} \label{equ:sec}
    (\xi_i - \phi(\xi_i))^T\lambda^{sec}_i\phi(\xi_i) \geqslant 0,\quad  \lambda^{\text{sec}}_i \in \mathbb{D}^{M}_{+}, \forall i \in \{1, 2\}.
\end{align}

If $\phi(\cdot)$ is slope-restricted, then for any $(k_1,\,k_2) \in \{1,\, \dots,\, M\}$ and $(i,j) \in \{1,\,2\}$ 
% \begin{small}
%  \begin{align} 
%   \begin{split}
%     ((\xi_i^\alpha - \xi_j^\beta)-(\phi(\xi_i^\alpha) - \phi(\xi_j^\beta)))\lambda^{slope, i}_{p,q}(\phi(\xi_i^\alpha) - \phi(\xi_j^\beta)) \geqslant 0,\\ \lambda^{\text{slope}, i,j}_{\alpha, \,\beta} \in \mathbb{R}^{N \times N}_{+}, \forall i,\, j \in \{1, 2\},~ \alpha, \, \beta \in {0,\, 1,\, \dots,\, \ell}
%     \end{split}
%  \end{align}
% \end{small}
 \begin{align}  \label{equ:slope}
    ({\beta}(\xi_i^{k_1} - \xi_j^{k_2})- & (\phi(\xi_i^{k_1}) - \phi(\xi_j^{k_2}))^T\lambda^{slope}_{i,j,k_1,k_2}(\phi(\xi_i^{k_1}) - \phi(\xi_j^{k_2})-\underline{\beta}(\xi_i^{k_1} -\xi_j^{k_2})) \geqslant 0,  
     ~ \lambda^{\text{slope}}_{i,j,k_1,k_2} \in \mathbb{R}_{+}^{2M},
 \end{align}
and if instead it is just monotonic then
\begin{align} 
      ({\beta}(\xi_i^{k_1} - \xi_j^{k_2})- & (\phi(\xi_i^{k_1}) - \phi(\xi_j^{k_2}))^T\lambda^{mon}_{i,j,k_1,k_2}(\phi(\xi_i^{k_1}) - \phi(\xi_j^{k_2})) \geqslant 0,  
     ~ \lambda^{\text{mon}}_{i,j,k_1,k_2} \in \mathbb{R}_{+}^{2M}.
 \end{align}
If $\phi(\cdot)$ is positive, then
\begin{align} \label{equ:pos}
  \lambda^{pos}_i\phi(\xi_i) \geqslant 0, \quad \lambda^{pos}_i \in \mathbb{R}^{1 \times M}_{+},  \forall i \in \{1, 2\}.
\end{align}
If $\phi(\cdot)$ satisfies the complimentary condition, then
\begin{align} \label{equ:comp}
    (\phi(\xi_i) - \xi_i)^T\lambda^{cpos}_i\phi(\xi_i) = 0, \quad  \lambda^{cpos}_i \in \mathbb{D}^{M} \forall i \in \{1, 2\}.
\end{align}
Note that \eqref{equ:comp} is a strict equality condition and that the positivity of the scaling terms $\lambda^{cpos}_i$ can be relaxed. If both $\phi(\cdot)$ and its complement are positive, then  the following cross terms are also obtained
\begin{subequations} \label{equ:crx}
\begin{align}
    \phi(\xi_i)^T\lambda^{crx}_{i,j}(\phi(\xi_j) - \xi_j)  \geqslant 0, \quad \lambda^{\text{crx}}_{i,j} \in \mathbb{R}_{+}^M, ~ \forall (i,\,j) \in \{1,\,2\},
    \\
    \phi(\xi_i)^T\lambda^{crx, \phi}_{i,j}\phi(\xi_j)  \geqslant 0, \quad  \lambda^{\text{crx,\,$\phi$}}_{i,j} \in \mathbb{R}_{+}^M, ~ \forall (i,\,j) \in \{1,\,2\}.
\end{align}\end{subequations}

\end{lemma}

\begin{definition}[QC for ReLU] \label{def:QC4ReLU}
If $\phi(\cdot) = ReLU(\cdot)$ then the matrix $\Lambda$ in \eqref{equ:Mlambda} is defined by
\begin{subequations} \label{qc:ReLU}
\begin{align}
[\Lambda_{13}, ~  \Lambda_{14}, ~  \Lambda_{15}]  & = [\lambda^{comp}_1, ~  -\lambda^{crx}_2  ~ , -\lambda^{cpos}_{2}],
\\
[\Lambda_{23}, ~  \Lambda_{24}, ~  \Lambda_{25}] &= [-\lambda^{crx}_1 , ~  \lambda^{comp}_2 ~ , -\lambda^{cpos}_{2}] ,
\\
    [\Lambda_{33}, ~  \Lambda_{34}, ~  \Lambda_{35}] &= [ -\lambda^{comp}_1 ~ , \lambda^{crx,\phi} + \lambda^{crx}_2+\lambda^{crx}_1,  ~  \lambda^{pos}_1 + \lambda^{cpos}_1],
\\
[\Lambda_{44}, ~  \Lambda_{45}, ~  \Lambda_{55}]  &= [-\lambda^{comp}_2, ~   \lambda^{pos}_{2}+\lambda^{cpos}_{2}, ~  0],
\end{align}
with the scaling terms $\lambda^{crx,\phi}, \lambda^{pos}_i, \lambda^{cpos}_i, \lambda^{crx}_i \in \mathbb{R}^{N \times N}_+$, 
$\lambda^{comp}_i \forall i \in \{1, 2\}$ defined in Lemma \ref{lem:ReLUQCs}.
\end{subequations}

% and for the sector and lope conditions $\mu_{sl}(\xi)^T \Lambda_s \mu_{sr}(\xi)\geq  0$ where

% \begin{subequations} \label{qc:sec&slope}

% \begin{align}
% \mu_{sl}(x) = \begin{bmatrix} \xi_1 - \phi{\xi_1} \\  \xi_2 - \phi{\xi_2} \\ (\xi_1^P - \xi_1^Q) - (\phi(\xi_1^P) - \phi(\xi_1^Q)) \\ (\xi_2^P - \xi_2^Q) - (\phi(\xi_2^P) - \phi(\xi_2^Q) \\ (\xi_1^P - \xi_2^Q) - (\phi(\xi_1^P) - \phi(\xi_2^Q)) \end{bmatrix},
% \mu_{sr}(x) = \begin{bmatrix} \phi{\xi_1} \\  \phi{\xi_2} \\ \phi(\xi_1^P) - \phi(\xi_1^Q) \\ \phi(\xi_2^P) - \phi(\xi_2^Q \\ \phi(\xi_1^P) - \phi(\xi_2^Q) \end{bmatrix},\\\Lambda_s = \diag(\lambda^{sec, 1}, \lambda^{sec, 2}, \lambda^{slope, 1}, \lambda^{slope, 2},\lambda^{slope, 1, 2}).
% \end{align}

% where $ \{\lambda^{sec, i}, \lambda^{slope, i}, \lambda^{slope, 1, 2}\} \in \mathbb{D}^{N}_+$.

% \end{subequations}

\end{definition}

\subsection*{Appendix 5: LMI quadratic constraints for quantisation function}
\begin{lemma}\label{lem:quant}
The quantisation function of \eqref{equ:q(x)} satisfies the following quadratic constraints: it is sector bounded
\begin{align}
   (x_i-q(x_i))^T  \lambda^\text{q,sec}_{i} (q(x_i) \geq 0, \quad \forall x_i \in \mathcal{X},\, i \in \{1, \, 2\}, ~ \lambda_\text{i}^{q,sec}  \in \mathbb{D}_+^{n_{x_i}} ,
\end{align}
and  can be both lower and upper bounded
\begin{align}
   & \lambda^\text{q,low}_{i} q(x_i)-(x_i-\Delta)) \geq 0, \quad \lambda^\text{q,up}_{i}(x_i+\Delta-q(x_i)) \geq 0,
    \\ & \quad \forall x_i \in \mathcal{X},\, i \in \{1, \, 2\}, ~ (\lambda^\text{q,low}_{i}, \,\lambda^\text{q,up}_{i} ) \in \mathbb{R}_+^{1 \times n_{x_i}} ,\nonumber
\end{align}
with these linear bounds implying the following quadratic for all $ (x_i,\,x_j) \in \mathcal{X}$
\begin{align}
   (q(x_i)-(x_i-\Delta))^T\lambda^\text{q,quad}_{i,j}((x_j+\Delta)-q(x_j)) \geq 0,
    ~, (i, \, j) \in \{1, \, 2\}, \lambda^\text{q,quad}_{i,j} \in \mathbb{D}_+^{n_{x_i}}.
\end{align} \label{equ: qx}
\end{lemma}

\begin{proof}
Follows from definition of the quantisation function (\ref{equ: qx}). These quadratics could be collected into the form  $P_q$
\end{proof}

\subsection*{Appendix 6: Proof of Theorem \ref{theorem:SDP}}

\begin{proof}
Multiplying \eqref{LMI} on the right by $\eta$ (defined in \eqref{mu}) and on the left by $\eta^T$  gives
\begin{align}\label{ineqs_proof}
 M_{\mathcal{X}}(P_{\mathcal{X}}) +  M_{q}(P_q)+M_\phi(\Lambda) + M_{\|f-g\|_2}(\Gamma) < 0,
\end{align}
where  $M_{\mathcal{X}}(P_{\mathcal{X}}),\,M_q(P_q)$ and $ M_\phi(\Lambda)$ are the non-negative inequalities defined in \eqref{equ:Mx}, \eqref{equ:Mq}, \eqref{equ:Mlambda} and where  $M_{\|f-g\|_2}(\Gamma) $ is the performance metric defined in \eqref{equ:Mgamma}. Since $M_{in}(P_{\mathcal{X}}),\,M_q(P_q)$ and $ M_\phi(\Lambda)$ are non-negative  whenever $x \in \mathcal{X}$, then  for \eqref{ineqs_proof} to hold, it must be that  $M_{\|f-g\|_2}(\Gamma) < 0$, giving the bound of \eqref{bound_thm}.
\end{proof}

\subsection*{Appendix 7: Numerical experiment tables and figures}

\begin{figure}[t]
\vspace{-1cm}
\floatconts
  {fig:QCompFig}
  {\caption{ Error bounds of the quantised neural networks.}}
  {%
    \subfigure[Quantisation error bound with $\Delta~=~2^{-2}$ and $\ell = 2$ with marked quantisation steps on the x-axis. ]{\label{fig:QCompFigBE}%
      \includegraphics[width=0.4\linewidth]{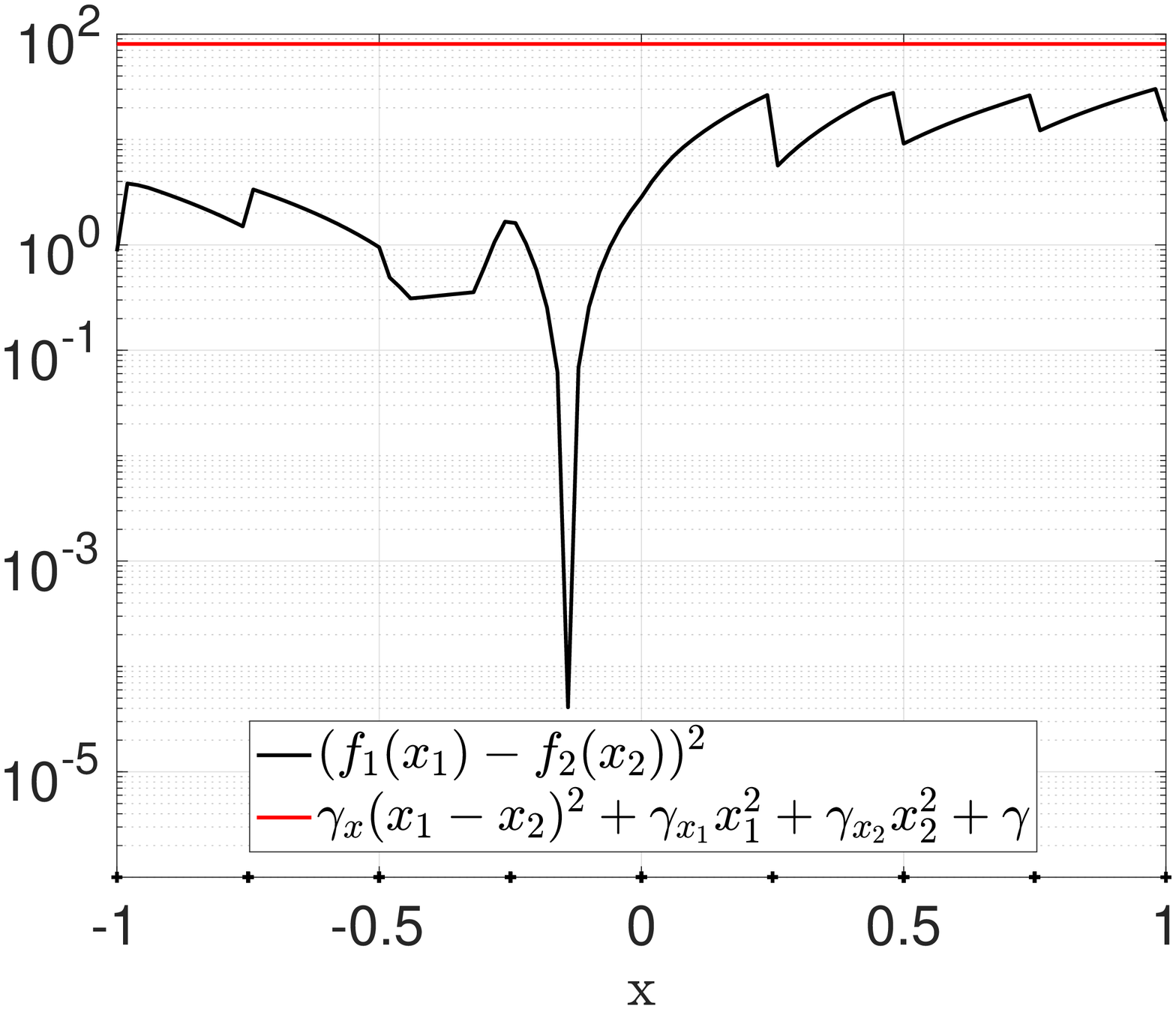}}%
    \qquad
    \subfigure[Quantisation error bound illustrated in 3D with  $\Delta~=~2^{-2}$ and  $\ell = 2$.]{\label{fig:QCompFig3D}%
      \includegraphics[width=0.42\linewidth]{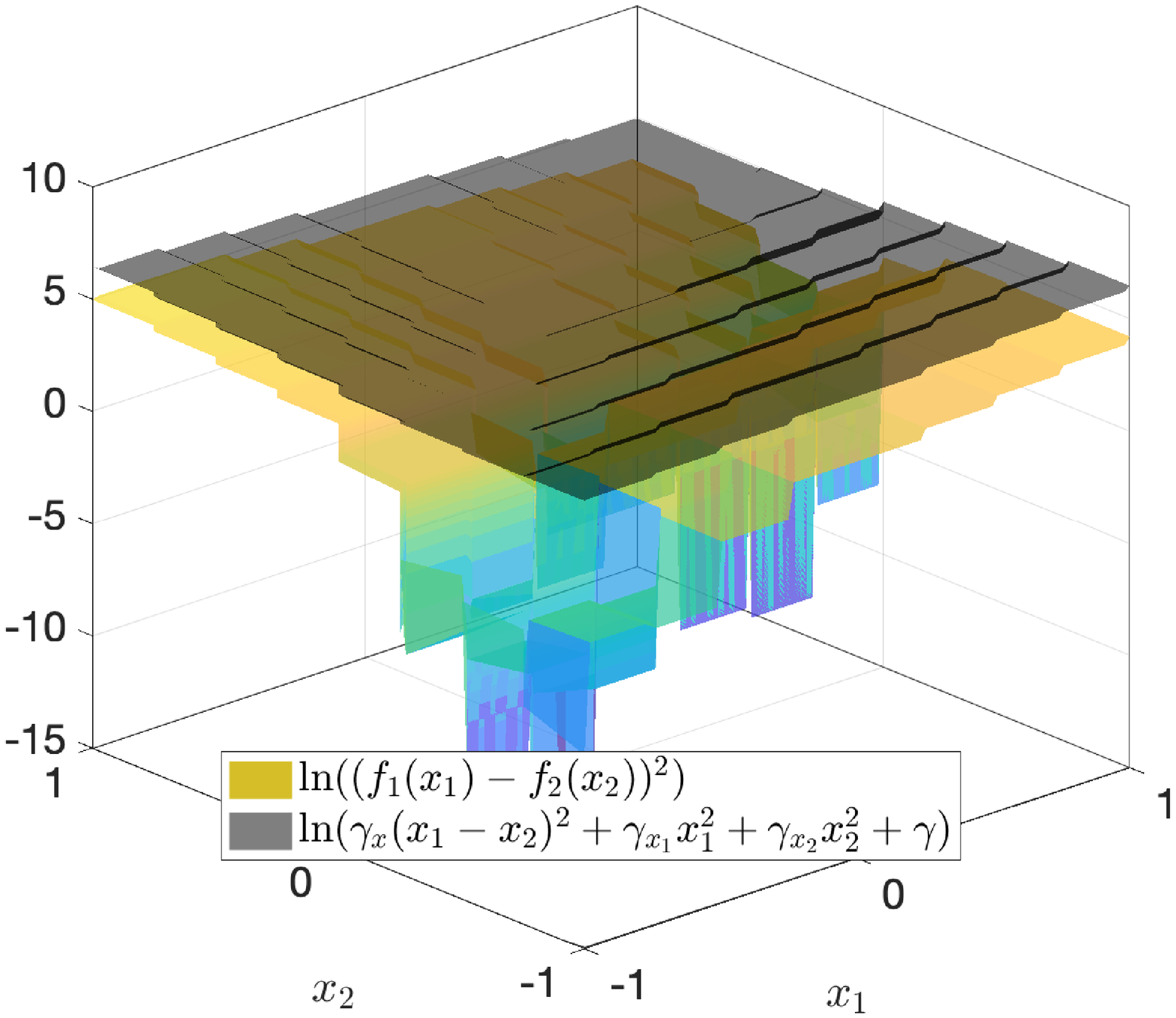}}
  }
\end{figure}

\begin{table} [H]
 \centering
 \begin{tabular}{|c||c|c|c|c|c|c|c|c|}
 \hline
  $\ell$ & $\gamma_x$ & $\gamma_{x_1}$ & $\gamma_{x_2}$ & $\gamma$ & mean($T$)  & max($T$) & min($T$) & Run time (s) \\ \hline
  1 & 14.9723 & 0.3027 & 0.0829 & 1.7490 & 3.3721 & 20.0648 & 1.0332 & 1.1667 \\
  2 & 181.1923 & 18.6963 & 11.9055 & 52.5948 & 5.0755 & 21.7170 & 2.3473 & 2.9527 \\
  3 & 1432.4 & 351.5557 & 259.4226 & 613.1431 & 6.0144 & 22.6888 & 3.3431 & 18.7965 \\
  4 & 16532 & 7309.7 & 5344.7 & 6775.3 & 7.1827 & 24.2606 & 4.2432 & 78.6121 \\
  \hline
 \end{tabular}
 \caption{\label{tab:NNcompData} Average values (over 100 runs) of the neural network error bounds from Theorem \ref{thm:SDP} with $T = \ln(f_1(x_1)-f_2(x_2))^2) - \ln(\gamma_x (x_1-x_2)^2 + \gamma_{x_1} {x_1}^2 + \gamma_{x_2} {x_2}^2 + \gamma)$.}
\end{table}

\begin{table} [H]
 \centering %\vspace{-0.5cm}
 \begin{tabular}{|c||c|c|c|c|c|c|c|c|}
  \hline
  $\ell$ & $\gamma_x$ & $\gamma_{x_1}$ & $\gamma_{x_2}$ & $\gamma$ & mean($T$)  & max($T$) & min($T$) & Run time (s) \\ \hline
  1 & 0.0000 & 1.0730 & 0.0705 & 2.3730 & 2.7206 & 8.7061 & 0.8257 & 0.8629 \\
  2 & 0.0000 & 19.9720 & 0.5333 & 27.7495 & 3.9042 & 10.3077 & 1.8512 & 3.4178 \\
  3 & 0.0006 & 190.4040 & 5.9258 & 558.7039 & 4.6004 & 10.2126 & 2.6929 & 19.8407 \\
  4 & 0.00870 & 1065.9 & 61.1359 & 10893 & 6.0101 & 11.5699 & 3.8971 & 87.7687 \\
  \hline
 \end{tabular}
 \caption{\label{tab:QcompData} Average values (over 100 runs) of the error bounds for the quantised neural network with $T = \ln(f_1(x_1)-f_2(x_2))^2) - \ln(\gamma_x (x_1-x_2)^2 + \gamma_{x_1} {x_1}^2 + \gamma_{x_2} {x_2}^2 + \gamma)$  with increasing number of neurons.}
 
\end{table}

\begin{figure}[H]
\vspace{-1cm}
    \centering 
    \includegraphics[width=0.35\linewidth]{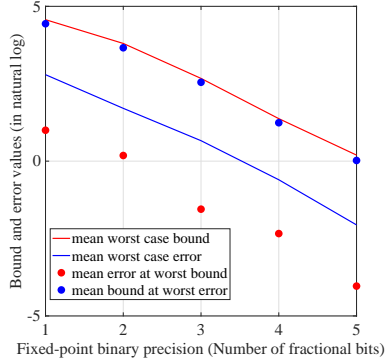}
    \caption{Progression of average (over 100 runs) worst case bound and error with quantisation level for randomly generated neural networks with 20 neurons and $\ell = 2$ layers.}
    \label{fig:WorstCaseQ}
\end{figure}

% \begin{figure}[H]
% \floatconts 
%   \label{fig:NNPruned} 
%   \caption{Error surface of the pruned neural network with $\ell = 4$.}
%     \subfigure[The error (coloured) and bound (grey) surfaces.] \label{fig:NNPrunedBE}
%       \includegraphics[width=0.43\linewidth]{Figures/PrunedBoundandErrorl4.eps} 
%     \qquad
%     \subfigure[The tightness surface $T$ of the bound and error surfaces in Fig. \ref{fig:NNPrunedBE}, viewed from another angle.]{\label{fig:NNPrunedT}
%       \includegraphics[width=0.40\linewidth]{Figures/PrunedTightnessl4.eps}
% \end{figure}

\begin{figure}[H]
\floatconts
  {fig:NNPruned}
  {\caption{Error surface of the pruned neural network with $\ell = 4$.}}
  {%
   \subfigure[The error (coloured) and bound (grey) surfaces.]{\label{fig:NNPrunedBE}%
      \includegraphics[width=0.45\linewidth]{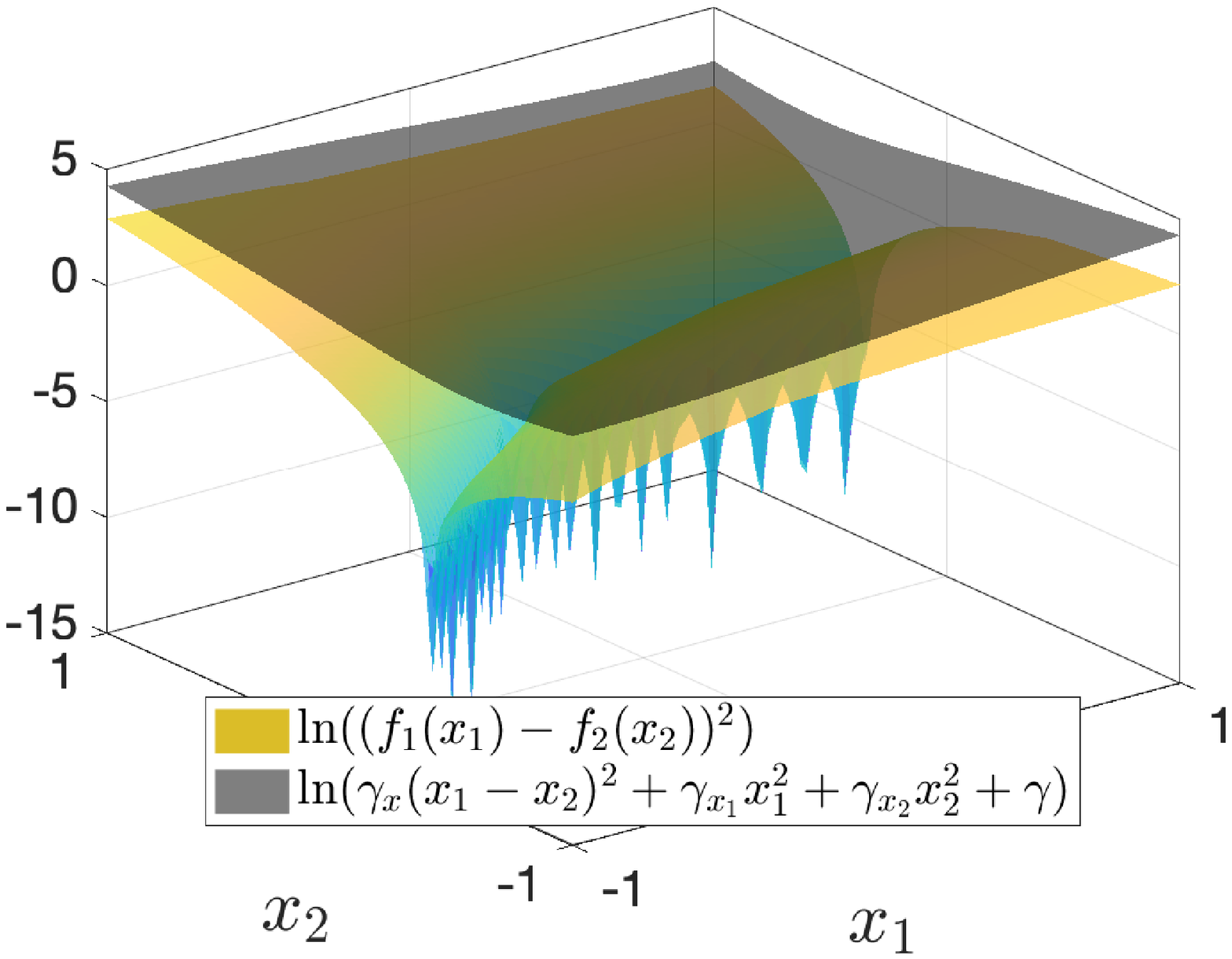}}%
    \qquad
    \subfigure[The tightness surface $T$ of the bound and error surfaces in Fig. \ref{fig:NNPrunedBE}, viewed from another angle.]{\label{fig:NNPrunedT}%
      \includegraphics[width=0.45\linewidth]{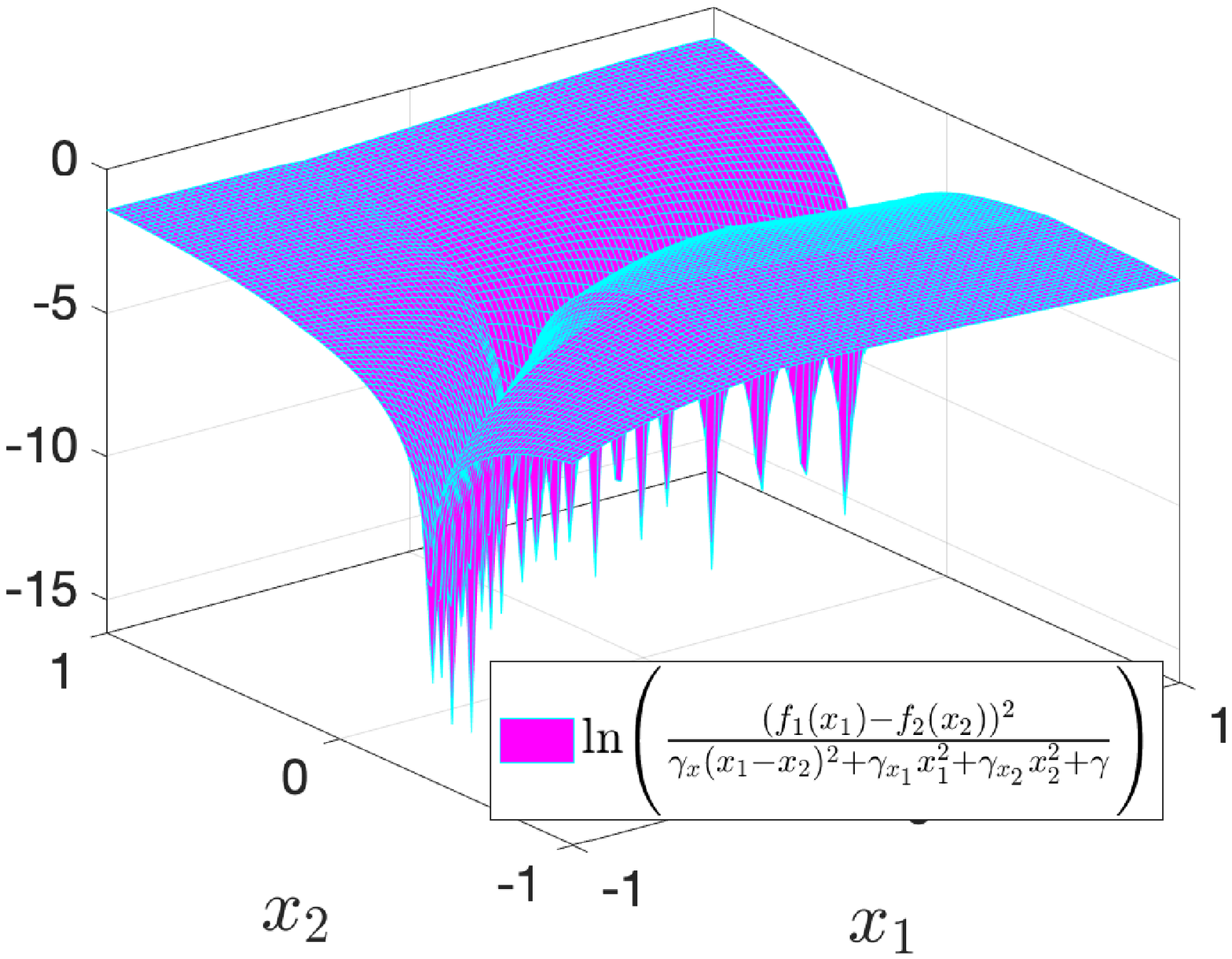}}
  }
\end{figure}

\end{document}